\PassOptionsToPackage{table}{xcolor}
\documentclass[]{alaya}
\usepackage{makecell}
\usepackage{wrapfig}
\usepackage{tabularx}
\usepackage{textcomp}
\usepackage{stfloats}
\usepackage{url}
\usepackage{verbatim}
\usepackage{titlesec}
\usepackage{tocloft}
\usepackage{adjustbox}
\usepackage{multirow}
\usepackage{pifont}
\usepackage[sc]{mathpazo}
\usepackage{tikz}
\usepackage{comment}
\usepackage{amsmath,amssymb}
\usepackage{colortbl}
\usepackage[numbers]{natbib}
\usepackage{color}
\usepackage{booktabs}
\usepackage{hyperref}
\usepackage{graphicx}
\usepackage{subcaption}
\RequirePackage{xspace}
\makeatletter
\DeclareRobustCommand\onedot{\futurelet\@let@token\@onedot}
\def\@onedot{\ifx\@let@token.\else.\null\fi\xspace}
\usepackage[most]{tcolorbox}
\usepackage{array}
\usepackage{siunitx}
\usepackage[table]{xcolor}
\usepackage{caption}
\definecolor{headerpurple}{HTML}{d8d2fc}
\definecolor{rowgray}{gray}{0.95}
\usepackage{CJKutf8}
\def\eg{\emph{e.g}\onedot}

\def\ie{\emph{i.e}\onedot}

\makeatother

\definecolor{adptorange}{RGB}{248, 205, 172}
\definecolor{cmpblue}{RGB}{189, 215, 238}

\definecolor{our_red}{RGB}{232,157,160}
\definecolor{our_blue}{RGB}{136,206,230}
\definecolor{our_orange}{RGB}{246,200,168}
\definecolor{our_green}{RGB}{178,211,164}

\definecolor{attn_code0}{RGB}{247,215,200}
\definecolor{attn_code1}{RGB}{238,169,139}
\definecolor{mlp_code0}{RGB}{204,201,221}
\definecolor{mlp_code1}{RGB}{102,95,153}
\definecolor{mygray}{HTML}{f0f0f0}

\definecolor{token_blue}{RGB}{84, 120, 140}

\usepackage{bbding}
\usepackage{fontawesome}
\usepackage{float}

\newlength\savewidth

\newcolumntype{x}[1]{>{\centering\arraybackslash}p{#1pt}}
\newcolumntype{y}[1]{>{\raggedright\arraybackslash}p{#1pt}}
\newcolumntype{z}[1]{>{\raggedleft\arraybackslash}p{#1pt}}

\renewcommand{\paragraph}[1]{\vspace{1.25mm}\noindent\textbf{#1}}

\usepackage{algorithm}
\usepackage{listings}

\definecolor{codeblue}{rgb}{0.25, 0.5, 0.5}
\definecolor{codekw}{rgb}{0.35, 0.35, 0.75}
\lstdefinestyle{Pytorch}{
    language = Python,
    backgroundcolor = \color{white},
    basicstyle = \fontsize{9pt}{8pt}\selectfont\ttfamily\bfseries,
    columns = fullflexible,
    aboveskip=1pt,
    belowskip=1pt,
    breaklines = true,
    captionpos = b,
    commentstyle = \color{codeblue},
    keywordstyle = \color{codekw},
}

\definecolor{green}{HTML}{009000}
\definecolor{red}{HTML}{ea4335}

\usepackage{algorithmic}
\usepackage{placeins}
\usepackage{amsthm}
\newtheorem{proposition}{Proposition}
\newtheorem{lemma}{Lemma}
\theoremstyle{remark}
\newtheorem*{remark}{Remark}


\title{SPW-Nav: A Streaming Panoramic World Model for Language-Guided Navigation}
\author[1,2,*]{Yunheng Liu}
\author[1,2,3,*]{Ziqi Cai}
\author[1,2]{Siqi Yang}
\author[2]{Yimu Wang}
\author[2]{Minggui Teng}
\author[1]{Jiaming Tan}
\author[2,\dagger]{Shuchen Weng}
\author[3]{Erwin Wu}
\author[1,\ddagger]{Kaipeng Zhang}
\author[2,\dagger]{Boxin Shi}
\affiliation[1]{Alaya Lab}
\affiliation[2]{Peking University}
\affiliation[3]{Institute of Science Tokyo}
\abstract{
Language-guided panoramic video generation benefits various downstream applications, such as interactive 3D scene exploration, virtual reality experiences, and embodied agent training. Existing panoramic generators follow predefined trajectories, and interactive world models act through low-level actions in perspective views. We propose SPW-Nav, a streaming panoramic world model that understands movement instructions and streams one minute of 2K 360\textdegree{} video in real time from a single panorama. SPW-Nav interprets each instruction in the previously generated panorama as camera motion. Spherical rotation decoupling applies rotation exactly on the sphere, pose-aligned conditioning keeps translation inputs bounded over long streams, and a multi-term memory with a few-step generator continues the scene as instructions change. We also build SPW-NavSet, panoramic videos with camera trajectories and verified instructions. Driven by language, SPW-Nav outperforms prior panoramic generators in camera-following accuracy and video quality, and supports on-the-fly instruction switching.
}

\github{\url{https://alaya-lab.github.io/SPW-Nav}}
\date{October 6, 2026}

\begin{document}
\maketitle
\begingroup
\renewcommand{\thefootnote}{\fnsymbol{footnote}}
\footnotetext[1]{Equal Contribution}
\footnotetext[3]{Project Leader}
\footnotetext[2]{Corresponding Authors}
\endgroup

\suppressfloats[t]
\begin{figure}[t]
\centering
\includegraphics[width=\linewidth]{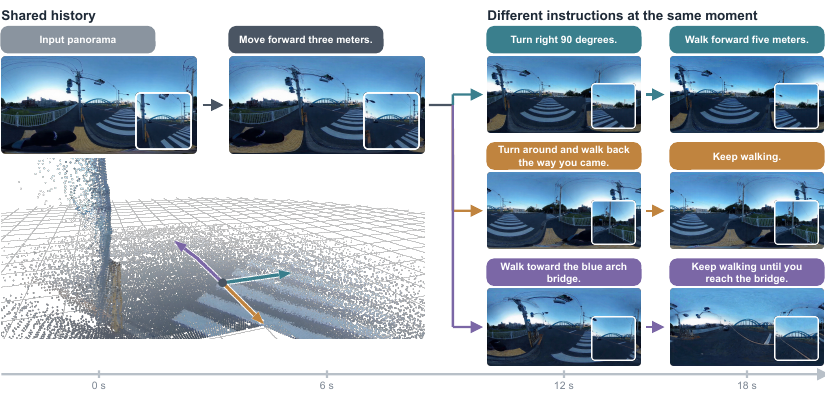}
\caption{
Language-guided exploration with SPW-Nav.
\textbf{\textit{Top left.}} a shared history generated from a single panorama. \textbf{\textit{Bottom left.}} a point cloud reconstructed from the generated frames, showing the motion directions of the three candidate routes in the scene. \textbf{\textit{Right.}} three different instructions given at the same moment continue this history, and each row shows the frames generated after its two instructions. Every frame shows the 360\textdegree{} panorama with its front view inset.
}
\label{fig:teaser}
\end{figure}

\section{Introduction}
\label{sec:introduction}

Generative world models let users and agents explore virtual environments, and exploration requires control over where to move and what to observe next. Panoramic video covers the surroundings in all directions, and natural language expresses movement without camera poses. Language-guided exploration therefore calls for continuous panoramic video that follows a sequence of instructions and stays consistent as the route changes.

Prior approaches address only parts of this interaction. Text-to-video models~\citep{ho2022video, singer2022make} condition an entire video on a single prompt. Camera-controlled panoramic generators~\citep{ji2025campvg, liu2026omniroam, li2026panoworld} follow predefined trajectories, while navigation and interactive world models~\citep{bar2024navigation, hafner2020dream, bruce2024genie, zhang2025matrixgame, he2025matrixgame2, mao2025yume} and vision-and-language agents~\citep{anderson2018vln, savva2019habitat} act through low-level actions (\eg, keyboard inputs), mostly in forward views. None of these video generators understands instructions that refer to scene content (\eg, walking to a door behind the camera), and this semantic gap motivates SPW-Nav.

Navigating a generated panoramic world raises three challenges. \textbf{\textit{(i) Scene-dependent navigation.}} The meaning of an instruction depends on the panorama the model itself generates and therefore cannot be turned into motion in advance. \textbf{\textit{(ii) Camera representation.}} A panoramic turn only rearranges visible content, yet existing generators synthesize it, and open-ended navigation drives world-frame poses beyond the training range. \textbf{\textit{(iii) Translation following.}} The camera pose is a weak signal next to rich visual context, and 2K streaming leaves few network evaluations per frame.

In this paper, we propose SPW-Nav, a \textbf{S}treaming \textbf{P}anoramic \textbf{W}orld model for language-guided \textbf{Nav}igation. As illustrated in Figure~\ref{fig:teaser}, SPW-Nav understands the semantics of movement instructions (\eg, walking to the door) and generates one minute of 360\textdegree{} video at 2K resolution (\(1024 \times 2048\)) and 26.6 FPS from a single panorama. \textbf{\textit{(i) Scene-aware navigation.}} A vision-language backbone reads the previously generated panorama and turns each instruction into an explicit displacement or a localized goal. \textbf{\textit{(ii) Decoupled camera representation.}} Spherical rotation decoupling applies the commanded rotation exactly by resampling the sphere, leaving only translation to the video backbone. Pose-aligned conditioning expresses this translation relative to the previous chunk and keeps the pose inputs bounded however long the navigation runs. \textbf{\textit{(iii) Enhanced translation following.}} Noise-tail trajectory learning focuses training and distillation on high noise levels, where the pose is the main motion cue, and opposite-trajectory guidance contrasts the commanded translation with its reverse at inference. A distilled pyramid generator with sparse attention~\citep{li2026solattn} and a multi-term memory sustain these one-minute streams. The explicit motion representation lets SPW-Nav learn navigation and generation from complementary data in SPW-NavSet.

Our contributions are summarized as follows.
\begin{itemize}
    \item We propose SPW-Nav, the first streaming panoramic world model for language-guided navigation, generating one-minute 2K 360\textdegree{} streams at 26.6 FPS. We also build SPW-NavSet, 233 hours of panoramic video with trajectories and verified instructions.
    \item We propose spherical rotation decoupling, which applies rotation exactly on the sphere so that the video backbone models translation alone, and pose-aligned conditioning, which keeps translation inputs bounded however far the navigation goes.
    \item We propose noise-tail trajectory learning and opposite-trajectory guidance to strengthen translation following, and enable streaming by making the video backbone autoregressive with a multi-term memory and distilling it to two steps per pyramid scale.
\end{itemize}

\section{Related work}
\label{sec:related_work}

\textbf{Panoramic video generation.} Spherical networks~\citep{cohen2018sphericalcnn, coors2018spherenet, su2017spherical} and 360\textdegree{} scene understanding~\citep{sun2021hohonet} process equirectangular imagery. Diffusion models generate panoramic images~\citep{zhang2024taming}, surround-view driving video~\citep{wen2024panacea}, and panoramic video~\citep{zhang2025panodit}, and have been extended to 360\textdegree{} video with latitude- and longitude-aware mechanisms~\citep{xia2025panowan} and to high-resolution cubemap autoregression~\citep{li2026cubecomposer}. Panoramic depth estimation~\citep{reyarea2022360monodepth} provides geometric priors, and video generation models in turn recover HDR environment maps of the full surrounding scene~\citep{liang2025luxdit, cai2026vlite}. None of them follows the camera motion of a navigation route, whereas SPW-Nav treats each equirectangular frame as a function on the sphere, on which rotation becomes exact resampling.

\textbf{Camera-controlled panoramic synthesis.} CamPVG~\citep{ji2025campvg} conditions panoramic video on camera motion through spherical epipolar attention, OmniRoam~\citep{liu2026omniroam} combines trajectory-controlled preview with temporal refinement, and PanoWorld~\citep{li2026panoworld} uses dense ray conditioning with geometry-aware memory. These methods take predefined trajectories rather than inferring motion from language, and GenDoP~\citep{zhang2025gendop} generates camera trajectories from text but no video. Perspective camera-controlled generators~\citep{he2024cameractrl, zheng2024cami2v, bai2025recammaster, zhou2025seva, huang2025voyager} follow trajectories within a limited field of view, where a turn must synthesize new content, and LiVER~\citep{cai2026liver} conditions generation on layout, lighting, and camera controls rendered from an explicit 3D scene proxy. SPW-Nav instead infers the movement from language in the previously generated panorama and retains context in a streaming memory as instructions change.

\textbf{Navigation and interactive world models.} World models predict future observations from past observations and actions, from latent dynamics~\citep{ha2018world, hafner2020dream, hafner2021mastering, hafner2023dreamerv3, micheli2023iris, schrittwieser2020muzero} to high-resolution video diffusion~\citep{blattmann2023align}. Autoregressive video diffusion with few-step distillation further enables streaming generation~\citep{chen2024diffusionforcing, yin2025causvid, huang2025selfforcing, yuan2026helios}. NWM~\citep{bar2024navigation} predicts egocentric observations from navigation actions, and Genie~\citep{bruce2024genie}, Matrix-Game~\citep{zhang2025matrixgame, he2025matrixgame2}, and Yume~\citep{mao2025yume} generate explorable perspective worlds controlled by latent or keyboard actions. Such actions prescribe camera motion without scene semantics and cannot express instructions about scene content (\eg, walking to a door). Streaming benchmarks~\citep{liu2026streamavbench} report drift and slow interactive responses. SPW-Nav instead interprets each instruction semantically, localizing referenced objects and places in the generated panorama, and learns navigation and generation from complementary data.

\section{SPW-NavSet}
\label{sec:dataset}

Real-world panoramic videos~\citep{tan2026alayavista} cover diverse scenes with estimated camera trajectories, while offering limited control over camera motion and few paired movement instructions. We therefore curate SPW-NavSet from filtered real outdoor videos and indoor videos rendered along controllable camera paths, and pair panoramas with movement instructions.

\textbf{Panoramic videos.} The video backbone is trained on panoramic videos with camera trajectories. These include indoor videos rendered with Habitat~\citep{savva2019habitat} from HM3D~\citep{ramakrishnan2021hm3d}, Replica~\citep{straub2019replica}, and ReplicaCAD~\citep{szot2021habitat} scenes with exact camera poses, as well as real outdoor street videos from MUGEN~\citep{tan2026alayavista} with estimated poses. We remove clips whose camera stays static or only rotates in place, and resample each frame to the orientation of the first one. The videos total about 18K clips and 233 hours at \(1024 \times 2048\) (appendix Figure~\ref{fig:dataset_samples}).

\textbf{Instruction annotations.} The vision-language backbone is trained on panoramas paired with movement instructions and motion targets. Explicit instructions (\eg, moving forward three meters or turning back) come from camera trajectories with varied phrasings. Goal instructions name a visible object or place (\eg, a door), localized by a bounding box and verified by vision-language models. We manually check 200 random samples and hold out scenes for validation and testing.

\textbf{SPW-NavBench.} This test set contains 100 real-world outdoor clips, disjoint from all training data down to the source videos. Each clip provides an initial panorama, its ground-truth camera trajectory, and two instructions. A motion instruction describes this trajectory, and a navigation instruction is either an explicit instruction or a goal instruction (see appendix).

\begin{figure}[t]
\centering
\includegraphics[width=\linewidth]{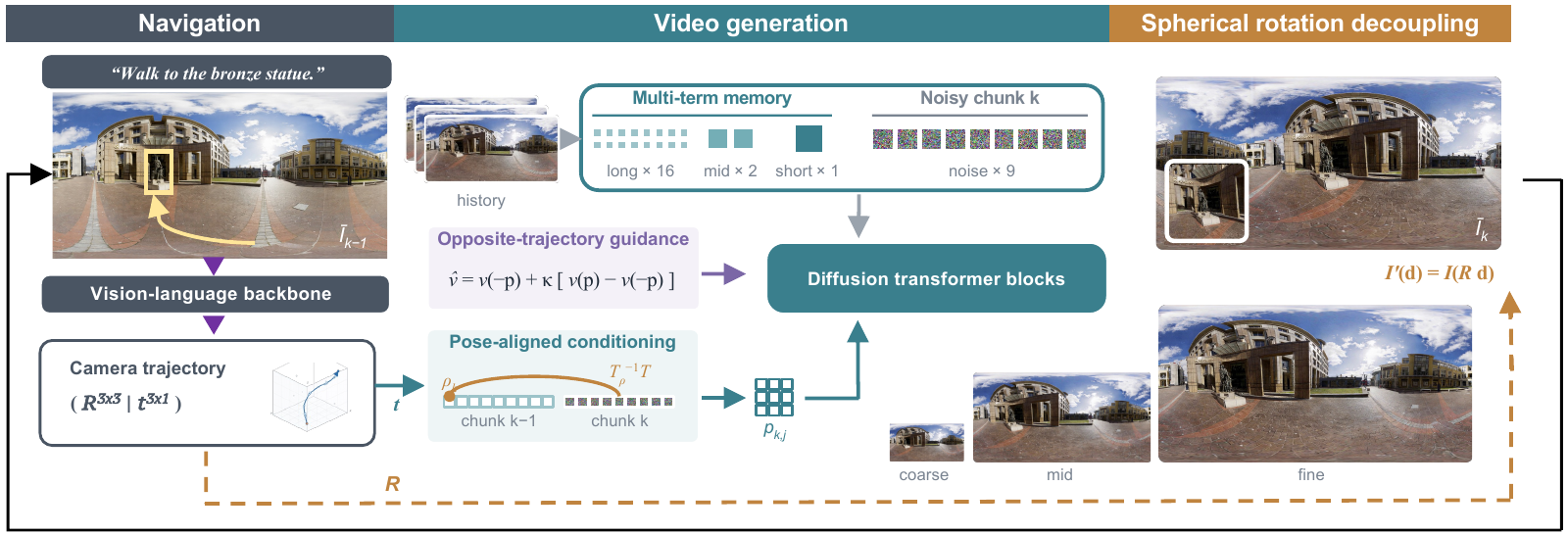}
\caption{
Overview of the SPW-Nav pipeline.
\textbf{Navigation.} The vision-language backbone localizes the referenced target (box) in the latest panorama \(\bar{I}_{k-1}\) and outputs a camera trajectory \((R, \mathbf{t})\).
\textbf{Video generation.} Memory tokens at long, mid, and short range and the noisy latents of chunk \(k\) enter the diffusion transformer, and pose tokens \(\mathbf{p}_{k,j}\) express \(\mathbf{t}\) relative to the reference frame \(\rho_k\). The chunk is denoised from coarse to fine, with guidance \(\hat{v}\) at the coarsest scale.
\textbf{Spherical rotation decoupling.} The rotation \(R\) bypasses the video backbone (dashed line) and is applied to the rotation-free output by spherical resampling. The last frame \(\bar{I}_k\) is fed back for the next instruction (solid line).
}
\label{fig:architecture}
\end{figure}

\section{Methodology}
\label{sec:methodology}

As illustrated in Figure~\ref{fig:architecture}, SPW-Nav generates each chunk of video from the latest panorama, the active instruction, and a memory of generated frames.

\subsection{Preliminaries}
\label{sec:method:prelim}

Given a panorama \(I_0\) at pose \(T_0\) and a sequence of movement instructions \(\{u_1, \ldots, u_K\}\), SPW-Nav outputs a stream of panoramic video segments \(\{V_1, \ldots, V_K\}\) that extends as new instructions arrive. Segment \(V_k\) carries out \(u_k\) and is generated in chunks of \(F = 33\) frames. Each frame is an equirectangular panorama \(I\) of resolution \(H \times W = 1024 \times 2048\), whose pixel \((x, y)\) observes the unit ray
\begin{equation}
\mathbf{d}(x, y) = (\cos\phi \sin\theta,\; \sin\phi,\; \cos\phi \cos\theta),
\qquad
\theta = 2\pi\Big(\tfrac{x + 1/2}{W} - \tfrac{1}{2}\Big),
\quad
\phi = \pi\Big(\tfrac{y + 1/2}{H} - \tfrac{1}{2}\Big),
\label{eq:erp}
\end{equation}
with longitude \(\theta\), latitude \(\phi\), and camera axes pointing right, down, and forward. Frame \(I_n\) of the stream has camera pose \(T_n = (R_n, \mathbf{t}_n) \in \mathrm{SE}(3)\).

\subsection{Navigation in the current view}
\label{sec:method:navigation}

Some instructions (\eg, moving forward or turning left) describe movement relative to the current viewpoint, while others (\eg, walking to the door) refer to content that only the panorama reveals. SPW-Nav therefore interprets each instruction \(u_k\) in the latest generated panorama \(\bar{I}_{k-1}\) (\ie, the last frame of the previous segment, with \(\bar{I}_0 = I_0\)) and expresses it as one of two motion targets relative to the current camera (Figure~\ref{fig:architecture}, left). A displacement target lists up to three waypoints \((x, z, \psi)\) in meters and degrees relative to the current camera, which covers multi-step movement, stopping, and corrections. A goal target marks the named object or place with a bounding box, which may wrap across the image border, and an approach distance \(r\). If \(c\) is the normalized horizontal center of the box, the camera turns to the bearing \(\psi^\star = 2\pi(c - 1/2)\) and then moves forward by \(r\). Either target is converted into the camera trajectory
\begin{equation}
A_k = \{T_{k,1}, \ldots, T_{k,L}\},
\qquad
T_{k,j} = T_{k-1} \cdot \Delta T_{k,j},
\qquad
\Delta T_{k,j} =
\begin{bmatrix}
\Delta R_{k,j} & \Delta\mathbf{t}_{k,j} \\
\mathbf{0}^\top & 1
\end{bmatrix}
\in \mathrm{SE}(3),
\label{eq:trajectory}
\end{equation}
where \(T_{k-1}\) is the camera pose at the end of the previous segment (\(T_0\) for the first segment), and \(\Delta T_{k,j}\) interpolates the waypoints over the segment.

\textbf{Architecture and training.} The vision-language backbone is Qwen3-VL-8B~\citep{bai2025qwen3vl} with rank-16 LoRA on the attention layers of its language model and a frozen vision encoder. It is trained on the instruction annotations of SPW-NavSet (Section~\ref{sec:dataset}) with the next-token objective on the motion target, conditioned on the panorama and the instruction.

\subsection{Pose-aligned video generation}
\label{sec:method:generation}

\textbf{Spherical rotation decoupling (SRD).} A camera rotation changes which direction each pixel observes but reveals no new content in a panorama (Figure~\ref{fig:architecture}, right). Rotating the camera by \(R\) maps a panorama \(I\) to \(I'(\mathbf{d}) = I(R\,\mathbf{d})\), evaluated at every pixel ray (Eq.~\ref{eq:erp}) with bilinear sampling that wraps around the longitude seam. In training, each frame is resampled to the orientation of the first frame, \(\widetilde{I}_n(\mathbf{d}) = I_n(R_n^\top R_0\,\mathbf{d})\), and the trajectory reduces to the translation \(\widetilde{\mathbf{t}}_n = R_0^\top(\mathbf{t}_n - \mathbf{t}_0)\). At inference, frames are generated in the orientation of the initial panorama and the rotation of \(A_k\) is restored after decoding. Rotation is thus exact up to interpolation (appendix Proposition~\ref{prop:srd}), leaving the generator to model only the parallax of translation.

\textbf{Pose-aligned conditioning (PAC).} World-frame poses of the same motion differ across places and grow with the distance traveled. PAC therefore anchors every pose to a recent view (Figure~\ref{fig:architecture}, middle). Every camera pose of chunk \(k\) is expressed as \(\widetilde{T}_n = T_{\rho_k}^{-1} T_n\) relative to \(\rho_k\), the first frame of the preceding chunk (\(\rho_1\) is the first frame of chunk \(1\)). In the rotation-free frame, this pose has an identity rotation block and translation \(s\,(\widetilde{\mathbf{t}}_n - \widetilde{\mathbf{t}}_{\rho_k})\) with a fixed scale \(s\). Since \(\rho_k\) lies in the memory (Section~\ref{sec:method:streaming}), the motion is stated relative to a visible view. Each of the \(9\) latent frames of the chunk receives its flattened \(3 \times 4\) pose \(\mathbf{p}_{k,j} \in \mathbb{R}^{12}\), whose identity rotation block is kept for generality, and \(\mathbf{p}_k\) stacks them. In every transformer block \(b\), a zero-initialized linear layer maps \(\mathbf{p}_{k,j}\) to a token offset
\begin{equation}
\mathbf{h}^{(b)}_{k,j,\ell} \leftarrow \mathbf{h}^{(b)}_{k,j,\ell} + W^{(b)} \mathbf{p}_{k,j} + \mathbf{c}^{(b)},
\label{eq:conditioning}
\end{equation}
which is added to every spatial token \(\ell\) of latent frame \(j\) and leaves the pretrained backbone unchanged at initialization. A relative translation of \((1,0,0)\) thus always means moving right in the current view. Since the reference is at most two chunks earlier, the inputs are invariant to rigid transformations of the world frame and stay within the training range.

\textbf{Noise-tail trajectory learning (NTL).} High noise levels largely determine camera motion and coarse structure~\citep{liu2026tars, zhao2026phyworld}. NTL therefore emphasizes them both when the autoregressive teacher learns the trajectory conditioning and when the few-step student is distilled (Section~\ref{sec:method:streaming}). In distillation, each student sample \(\hat{z}\) is re-noised to \(\hat{z}_\sigma = (1 - \sigma)\hat{z} + \sigma \epsilon\) and scored by a critic and the teacher. In both stages, a fraction \(\lambda = 0.8\) of the noise levels comes from the tail \([\sigma_h, \sigma_{\max}] = [0.9, 0.98]\) and the rest from the default distribution \(p_0\) of that stage,
\begin{equation}
\sigma \sim p_\lambda = \lambda\, \mathcal{U}[\sigma_h, \sigma_{\max}] + (1 - \lambda)\, p_0 .
\label{eq:loss}
\end{equation}
At high noise the sample shows almost no layout. The training target and the score difference thus mainly concern where content should move under the relative poses. At low noise, in contrast, the layout is already visible and the pose adds little information. Beyond improving navigation, NTL also speeds up the convergence of trajectory learning.

\textbf{Opposite-trajectory guidance (OTG).} In a panorama, moving forward and moving backward produce mirrored flow fields, with expansion at the heading and contraction at its antipode. OTG therefore contrasts the commanded translation with its reverse at the coarsest scale of the generator (Section~\ref{sec:method:streaming}), where the velocity is
\begin{equation}
\hat{v} = v_\theta(z_\sigma, \sigma, \mathcal{M}_k, -\mathbf{p}_k) + \kappa \big( v_\theta(z_\sigma, \sigma, \mathcal{M}_k, \mathbf{p}_k) - v_\theta(z_\sigma, \sigma, \mathcal{M}_k, -\mathbf{p}_k) \big),
\label{eq:guidance}
\end{equation}
where \(-\mathbf{p}_k\) negates the translation and keeps all other inputs. The difference cancels components shared by both directions (\eg, scene appearance) and amplifies the one that depends on the direction of travel. Finer scales refine appearance with the plain conditional velocity, and we use \(\kappa = 2\).

\subsection{Streaming generation with memory}
\label{sec:method:streaming}

When a new instruction arrives, the model must continue from the previous chunk without losing visual context and with few network evaluations per chunk. We build the streaming generator on the autoregressive recipe of Helios~\citep{yuan2026helios} and detail the loop in Algorithm~\ref{alg:streaming}.

\textbf{Multi-term memory.} The memory \(\mathcal{M}_k\) (Figure~\ref{fig:architecture}, middle) holds the generated stream in three terms, the most recent latent frame, the two before it, and sixteen earlier ones, with older terms patchified more coarsely. A fixed anchor also keeps the first frame of the stream, which keeps the initial view visible throughout the stream. Memory tokens enter the transformer alongside the chunk tokens at noise level zero, and the first chunk starts from the input panorama repeated as a static clip. Since generation runs in the rotation-free orientation and each remembered panorama covers all directions, rotation never disturbs the memory and recent appearance in every direction stays available.

\textbf{Few-step generation.} Each chunk is denoised from coarse to fine over a spatial pyramid of three scales, which splits the noise range into equal thirds. The student uses two steps per scale (four for the first chunk). After the first chunk, each chunk takes six network evaluations plus two for the guidance of Eq.~\ref{eq:guidance}. The student is initialized by ODE regression onto the teacher and trained with distribution matching distillation~\citep{yin2024dmd}, with the teacher as the real score and a critic tracking the student.

\textbf{Streaming training.} At inference the memory holds generated frames, whose errors differ from those of real video. ODE regression therefore uses histories from the model's own rollouts, and distillation corrupts ground-truth histories with noise, downsampling, and saturation shifts with probability \(0.9\), such that the student learns to continue from imperfect memory. Chunks do not overlap, and the memory alone carries continuity across chunks and instructions.

\section{Experiments}
\label{sec:experiments}

\subsection{Experimental setup}
\label{sec:exp:setup}

\textbf{Implementation details.}  The video backbone is Wan2.2-TI2V-5B~\citep{wan2025wan} fine-tuned with LoRA on panoramic video and produces \(1024 \times 2048\) video with a playback rate of 16 FPS. It is trained without language on the rotation-free videos of SPW-NavSet in autoregressive, pyramid, and trajectory-conditioned stages, and then distilled into a few-step student (see appendix).

\textbf{Baselines.}  We compare with generators that accept camera motion, each through its own interface. The baselines are NWM~\citep{bar2024navigation}, CubeComposer~\citep{li2026cubecomposer}, OmniRoam~\citep{liu2026omniroam}, and the Action and Causal variants of PanoWorld~\citep{li2026panoworld}, denoted PanoWorld-A.\ and PanoWorld-C. CubeComposer completes a perspective video into a panorama and takes neither trajectory nor language. It therefore appears only in camera-controlled generation. For camera-controlled and long video generation, we add Matrix-Game 3.0~\citep{wang2026matrixgame3}, a streaming world model with long-horizon memory driven by keyboard and mouse actions in a perspective view. Matrix-Game 3.0 is evaluated in its native view and as a panorama completed by CubeComposer. The comparison with perspective generators Stable Virtual Camera~\citep{zhou2025seva}, ReCamMaster~\citep{bai2025recammaster}, CamI2V~\citep{zheng2024cami2v}, HunyuanWorld-Voyager~\citep{huang2025voyager}, and CameraCtrl~\citep{he2024cameractrl} appears in the appendix.

\textbf{Metrics.}  Camera-following accuracy is measured on the camera motion estimated from the generated video~\citep{wang2025vggt} by the trajectory direction error (TDE), the cumulative and per-interval rotation errors (RE and RPE\textsubscript{R}), the error of the mean cumulative rotation (\(\Delta\)Rot.), a scale-free trajectory shape error (ATE), and motion accuracy (MA). Video quality is measured by PSNR, LPIPS, and FVD. The appendix gives the exact definitions.

\subsection{Language-guided navigation}
\label{sec:exp:e2e}

We first evaluate the full system, from a language instruction to panoramic video, and start every method from the same motion instruction of each SPW-NavBench clip. No prior panoramic generator understands movement instructions semantically, which is the gap SPW-Nav addresses. Each baseline therefore receives the instruction through the closest interface it offers. PanoWorld and OmniRoam take it in their text prompt with the motion preset matching its keywords, and NWM takes the actions planned by our vision-language backbone without fine-tuning.

\begin{table}[t]
\centering
\caption{
Language-guided navigation on SPW-NavBench. Each method receives the same instruction through its own interface. The italic row uses the base vision-language model and is not ranked. In all tables, the best is bold and the second best underlined.
}
\label{tab:e2e}
\footnotesize
\setlength{\tabcolsep}{2pt}
\begin{tabular*}{\linewidth}{@{\extracolsep{\fill}}lccccccccc@{}}
\toprule
Method & \(\Delta\)Rot.\,\(\downarrow\) & TDE\,\(\downarrow\) & RE\,\(\downarrow\) & RPE\textsubscript{R}\,\(\downarrow\) & ATE\,\(\downarrow\) & MA\,\(\uparrow\) & PSNR\,\(\uparrow\) & LPIPS\,\(\downarrow\) & FVD\,\(\downarrow\) \\
\midrule
NWM~\citep{bar2024navigation}          & 82.74 & 84.79 & 83.87 & 19.18 & 23.48 & 0.064 & 9.90 & N/A & 1216.1 \\
OmniRoam~\citep{liu2026omniroam}       & 13.00 & \underline{62.07} & 13.46 & 1.603 & 4.75 & 0.256 & \underline{11.47} & 0.654 & 694.9 \\
PanoWorld-C.~\citep{li2026panoworld} & 13.19 & 63.36 & 13.22 & \underline{1.517} & \underline{4.25} & \underline{0.264} & 11.29 & 0.645 & 638.5 \\
PanoWorld-A.~\citep{li2026panoworld} & \underline{12.76} & 63.03 & \underline{12.99} & 1.545 & \textbf{4.13} & 0.257 & 11.14 & \underline{0.633} & \underline{628.8} \\
\midrule
\textbf{SPW-Nav} & \textbf{1.45} & \textbf{56.66} & \textbf{3.02} & \textbf{0.780} & 6.28 & \textbf{0.305} & \textbf{15.69} & \textbf{0.398} & \textbf{351.6} \\
\textit{W/o navigation fine-tuning} & 0.79 & 100.74 & 3.84 & 1.379 & 10.34 & 0.028 & 15.23 & 0.413 & 410.7 \\
\bottomrule
\end{tabular*}
\end{table}

Table~\ref{tab:e2e} shows that SPW-Nav leads on every measure except ATE, where the keyword-matched presets give PanoWorld a predefined trajectory shape while SPW-Nav infers the trajectory from the instruction and the scene. SPW-Nav reaches a TDE of \(56.66^\circ\) against at least \(62.07^\circ\) and an FVD of 351.6 against at least 628.8, and its RE is less than one quarter of that of any baseline, since SRD applies the interpreted rotation exactly while the baselines rotate only slightly. Both PanoWorld variants also lose about 5 dB of PSNR relative to Table~\ref{tab:main}, whereas SPW-Nav loses less than 1 dB. Replacing the fine-tuned backbone with its base model raises TDE to \(100.74^\circ\) and lowers MA to 0.028, since the base model often misreads the commanded direction.

\textbf{Navigation accuracy.} \label{sec:exp:language}
On the navigation instructions of SPW-NavBench (40 explicit and 60 goal instructions, see appendix), the vision-language backbone interprets all explicit instructions correctly and localizes 62\% of the visible targets. On the held-out test scene, fine-tuning raises the accuracy from 0.663 to 0.925 for explicit instructions and from 0.154 to 0.692 for goal instructions. Without the panorama, goal localization drops to zero, which indicates that goals are located in the image rather than guessed from the words. On the motion instructions, it matches the commanded direction for 98\% of the instructions with a mean distance error of 0.03\,m relative to the commanded distance.

\subsection{Camera-controlled generation}
\label{sec:exp:trajectory}

We then isolate the video backbone by giving every method the ground-truth camera trajectory instead of language. SPW-Nav and OmniRoam take the full trajectory, the PanoWorld variants take its translation with a fixed heading, NWM takes the corresponding actions with three ground-truth past views, and CubeComposer takes a perspective video rendered along the commanded rotation.

\begin{table}[t]
\centering
\caption{
Camera-controlled panoramic generation from ground-truth trajectories.
}
\label{tab:main}
\footnotesize
\setlength{\tabcolsep}{2pt}
\begin{tabular*}{\linewidth}{@{\extracolsep{\fill}}lccccccccc@{}}
\toprule
Method & \(\Delta\)Rot.\,\(\downarrow\) & TDE\,\(\downarrow\) & RE\,\(\downarrow\) & RPE\textsubscript{R}\,\(\downarrow\) & ATE\,\(\downarrow\) & MA\,\(\uparrow\) & PSNR\,\(\uparrow\) & LPIPS\,\(\downarrow\) & FVD\,\(\downarrow\) \\
\midrule
NWM~\citep{bar2024navigation}          & 82.55 & 86.19 & 85.64 & 19.12 & 22.89 & 0.081 & 10.68 & N/A & 1176.1 \\
CubeComposer~\citep{li2026cubecomposer} & 13.97 & 91.40 & 13.97 & 1.521 & 22.60 & 0.073 & 14.61 & 0.565 & 973.1 \\
OmniRoam~\citep{liu2026omniroam}       & \underline{12.30} & 77.12 & 15.07 & 2.042 & 10.40 & 0.111 & 15.36 & 0.489 & 610.5 \\
PanoWorld-C.~\citep{li2026panoworld} & 13.46 & 26.28 & \underline{13.47} & 1.528 & \underline{4.24} & \underline{0.688} & 16.12 & 0.460 & 462.0 \\
PanoWorld-A.~\citep{li2026panoworld} & 13.61 & \textbf{24.26} & 13.62 & 1.530 & \textbf{3.93} & \textbf{0.726} & \underline{16.24} & \underline{0.425} & \underline{384.2} \\
Matrix-Game 3.0 + CubeComposer & 13.13 & 92.91 & 13.14 & \underline{1.479} & 20.24 & 0.080 & 14.59 & 0.587 & 892.8 \\
\midrule
\textbf{SPW-Nav} & \textbf{0.42} & \underline{24.30} & \textbf{1.84} & \textbf{0.282} & 5.00 & 0.631 & \textbf{16.48} & \textbf{0.382} & \textbf{308.0} \\
\bottomrule
\end{tabular*}
\end{table}

As Table~\ref{tab:main} shows, SPW-Nav rotates by \(13.59^\circ\) for a commanded rotation of \(14.01^\circ\), while the panoramic baselines rotate by less than \(2^\circ\) and NWM overshoots. Its RE is thus less than one seventh of that of PanoWorld-A. This margin holds by construction, since SRD applies the turn by resampling, whereas the baselines either fix the heading or reproduce little of the rotation they receive. For translation, SPW-Nav matches PanoWorld-A.\ in direction, while PanoWorld-A.\ reaches a lower ATE and a higher MA. PanoWorld-A.\ specializes in fixed-heading 960\(\times\)480 clips, whereas SPW-Nav streams 2K panoramas with arbitrary turns and a few-step generator. This gap therefore reflects a trade-off between specialization and generality. SPW-Nav achieves the best video quality on all three measures, and Figure~\ref{fig:qual_trajectory} shows qualitative comparisons. Against the perspective generators, evaluated on a perspective view cropped from the panorama (see appendix), SPW-Nav reaches the lowest RE and RPE\textsubscript{R} and the second-best PSNR, LPIPS, and FVD.

\begin{figure}[t]
\centering
\includegraphics[width=\linewidth]{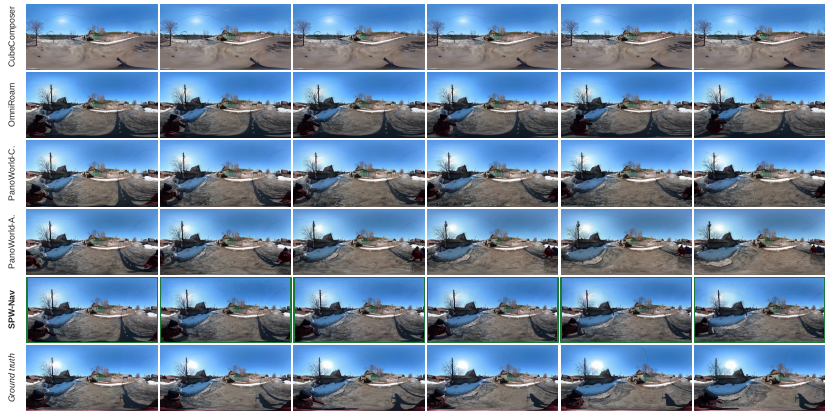}
\caption{
Qualitative comparison with ground-truth trajectories.
Each row shows the panoramic video generated by one method at six evenly spaced timestamps from left to right. The last row shows the ground truth, and SPW-Nav is outlined in green.
}
\label{fig:qual_trajectory}
\end{figure}
\subsection{Applications}
\label{sec:exp:streaming}
\label{sec:exp:loop}

\textbf{Long video generation.} SPW-Nav streams one minute of continuous 2K panoramic video (960 frames at 16 FPS) from a single panorama, including panoramas outside its training data (appendix Figures~\ref{fig:zeroshot} and~\ref{fig:longvideo}). Such long streams must stay consistent when a route returns to its start. For each of 20 scenes, we run two closed routes (\ie, the final pose equals the initial one) of twelve chunks (see appendix). Memory gain measures how much closer a revisited view is to its earlier observation than a non-revisiting view, and closure errors compare the final and initial poses. Table~\ref{tab:loop} compares SPW-Nav with OmniRoam, PanoWorld-A., HunyuanWorld-Voyager, and Matrix-Game 3.0, where the last two are evaluated in their native perspective view. SPW-Nav closes the loop with a rotation error of \(1.62^\circ\), against \(8.91^\circ\) for HunyuanWorld-Voyager, the next best, and up to \(39.52^\circ\) for PanoWorld-A. It also reaches the lowest translation closure error, 30.7\%. Matrix-Game 3.0 reaches a higher memory gain than SPW-Nav (0.249 in its perspective view and 0.204 as a panorama, against 0.173), but its closure rotation error exceeds \(22^\circ\) in both settings.

\begin{table}[t]
\centering
\caption{
Long-horizon consistency on 40 closed routes of twelve chunks. \(^\dagger\)Native perspective view.
}
\label{tab:loop}
\small
\setlength{\tabcolsep}{2pt}
\begin{tabular*}{\linewidth}{@{\extracolsep{\fill}}lccc@{}}
\toprule
Method & Memory gain $\uparrow$ & Closure rot. (\(^\circ\)) $\downarrow$ & Closure trans. (\%) $\downarrow$ \\
\midrule
OmniRoam~\citep{liu2026omniroam} & 0.026 & 11.02 & 57.5 \\
PanoWorld-A.~\citep{li2026panoworld} & 0.084 & 39.52 & 59.0 \\
HunyuanWorld-Voyager\(^\dagger\)~\citep{huang2025voyager} & 0.087 & \underline{8.91} & 40.7 \\
Matrix-Game 3.0\(^\dagger\)~\citep{wang2026matrixgame3} & \textbf{0.249} & 28.41 & 32.7 \\
\quad + CubeComposer~\citep{li2026cubecomposer} & \underline{0.204} & 22.26 & \underline{32.0} \\
\midrule
\textbf{SPW-Nav} & 0.173 & \textbf{1.62} & \textbf{30.7} \\
\bottomrule
\end{tabular*}
\end{table}

\textbf{Instruction switching.} New instructions may arrive at any time. We run 20 scenes through four consecutive instructions spanning twelve chunks (about 25 seconds), where the vision-language backbone reads the last generated frame before every instruction and the memory holds only generated frames. Figure~\ref{fig:streaming} shows examples, and Table~\ref{tab:streaming} lists per-segment metrics. Instruction-following metrics fluctuate across segments, while FVD is higher after the first segment. Switches cause no visible jump, as inter-frame LPIPS across instruction boundaries is no higher than within segments. With eight H200 GPUs and sparse attention, SPW-Nav generates 26.6 frames per second, faster than the 16 FPS playback rate of its output.

\begin{figure}[t]
\centering
\includegraphics[width=\linewidth]{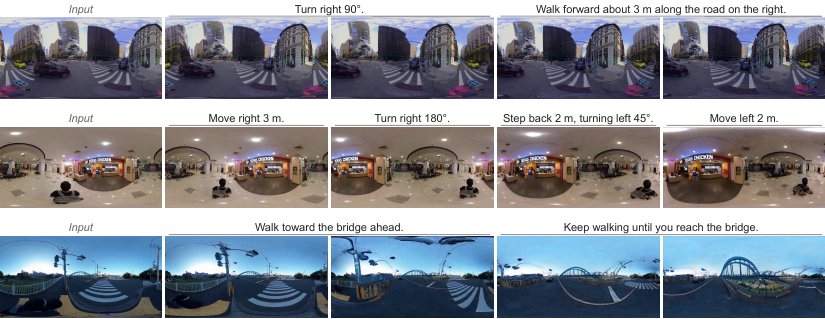}
\caption{
Streaming generation with instruction switching. The first column shows the input panorama, and each instruction above spans the frames generated while following it.
}
\label{fig:streaming}
\end{figure}

\subsection{Ablation study}
\label{sec:exp:analysis}

Table~\ref{tab:ablation} ablates NTL, OTG, and SRD. Without NTL, neither the teacher nor the student uses the noise tail. Disabling OTG means \(\kappa = 1\) in Eq.~\ref{eq:guidance}.

\textbf{Impact of OTG.} With NTL but without OTG, the video often moves in an unintended direction. Adding OTG cuts TDE from \(43.53^\circ\) to \(24.30^\circ\), raises motion accuracy, and lowers FVD. Larger guidance weights offer little further improvement in TDE but raise RE and degrade appearance (appendix Table~\ref{tab:guidance}). OTG adds only two network evaluations per chunk, both at the coarsest scale.

\textbf{Impact of NTL.} With OTG disabled, removing NTL raises RE from \(0.99^\circ\) to \(26.58^\circ\), PSNR drops by more than 3~dB, and FVD rises from 395.6 to 512.2. OTG then lowers TDE only from \(29.15^\circ\) to \(27.85^\circ\) and further degrades appearance. NTL is therefore needed both for motion accuracy and for the gain from OTG.

\textbf{Impact of SRD.} Without SRD, the model is trained on videos that keep their rotation and generates the rotation itself. This variant is also trained without NTL. Relative to the full model, its RE rises from \(1.84^\circ\) to \(31.31^\circ\), and its FVD from 308.0 to 596.5. This variant has worse rotation accuracy and video quality, although the comparison does not isolate SRD from NTL.

\begin{table}[t]
\centering
\caption{
Ablation on SPW-NavBench. The upper block varies NTL and OTG with SRD enabled. The lower row disables both SRD and NTL.
}
\label{tab:ablation}
\small
\setlength{\tabcolsep}{2pt}
\begin{tabular*}{\linewidth}{@{\extracolsep{\fill}}cccccccccc@{}}
\toprule
NTL & OTG & TDE\,\(\downarrow\) & MA\,\(\uparrow\) & ATE\,\(\downarrow\) & RE\,\(\downarrow\) & RPE\textsubscript{R}\,\(\downarrow\) & PSNR\,\(\uparrow\) & LPIPS\,\(\downarrow\) & FVD\,\(\downarrow\) \\
\midrule
 & & 29.15 & 0.469 & \textbf{4.80} & 26.58 & 3.157 & 13.69 & 0.494 & 512.2 \\
 & \checkmark & \underline{27.85} & \underline{0.517} & 5.64 & 28.13 & 3.350 & 13.26 & 0.514 & 562.3 \\
\checkmark & & 43.53 & 0.395 & 12.33 & \textbf{0.99} & \textbf{0.160} & \textbf{17.05} & \textbf{0.361} & \underline{395.6} \\
\checkmark & \checkmark & \textbf{24.30} & \textbf{0.631} & \underline{5.00} & \underline{1.84} & \underline{0.282} & \underline{16.48} & \underline{0.382} & \textbf{308.0} \\
\midrule
\multicolumn{2}{l}{\textit{W/o SRD}} & 27.99 & 0.521 & 5.99 & 31.31 & 3.801 & 12.97 & 0.518 & 596.5 \\
\bottomrule
\end{tabular*}
\end{table}

\section{Conclusion}
\label{sec:conclusion}

We present SPW-Nav, a streaming panoramic world model that interprets each language instruction in the previously generated panorama and streams one minute of 2K 360\textdegree{} video at 26.6 FPS. Spherical rotation decoupling makes rotation exact, pose-aligned conditioning keeps translation inputs bounded, and noise-tail trajectory learning with opposite-trajectory guidance strengthens translation following. SPW-Nav outperforms existing panoramic generators in camera-following accuracy and video quality and closes long routes with near-exact rotation, making language a practical interface for exploring generated worlds. Since generation runs faster than playback, new instructions can be issued while the stream plays. Consistency on distant revisits and the trajectory shape of translation remain open. In future work, a pose-indexed 3D memory could make distant revisits more consistent, and more diverse training trajectories could improve the shape of translation. Motion targets could also come from clicks on the panorama or external planners, since the video backbone receives camera trajectories rather than language.

\section*{Acknowledgments}
This study was carried out using the TSUBAME4.0 supercomputer at Institute of Science Tokyo.

\clearpage
\beginappendix
\counterwithin{figure}{section}
\counterwithin{table}{section}
\counterwithin{algorithm}{section}
\renewcommand{\thefigure}{\thesection\arabic{figure}}
\renewcommand{\thetable}{\thesection\arabic{table}}

\section{Additional results}
\label{app:results}
Upward and downward arrows mark metrics where larger and smaller values are better.

\FloatBarrier
\subsection{Zero-shot generalization to unseen scenes}
Figure~\ref{fig:zeroshot} shows streams generated from panoramas that are not part of the training data, covering indoor halls, city streets, waterfronts, and natural scenes.

\begin{figure}[!ht]
\centering
\includegraphics[width=\linewidth]{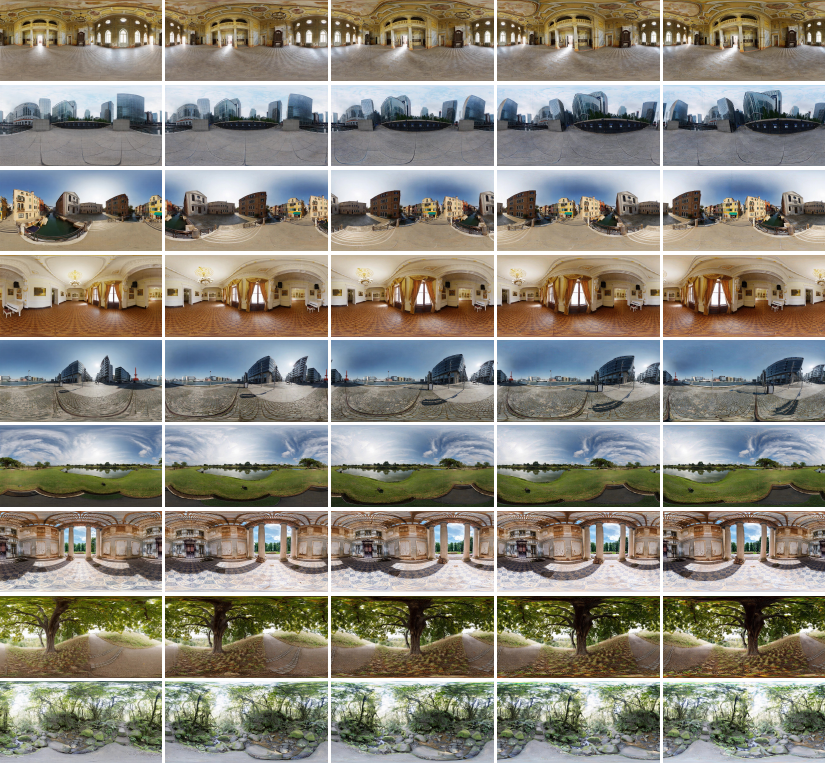}
\caption{
Zero-shot streams generated by SPW-Nav from panoramas outside its training data. Each row is one stream of 14 to 30 seconds with evenly spaced frames.
}
\label{fig:zeroshot}
\end{figure}

\FloatBarrier
\subsection{One-minute streams}
Figure~\ref{fig:longvideo} shows streams of about one minute, each generated continuously from a single panorama.

\begin{figure}[!ht]
\centering
\includegraphics[width=\linewidth]{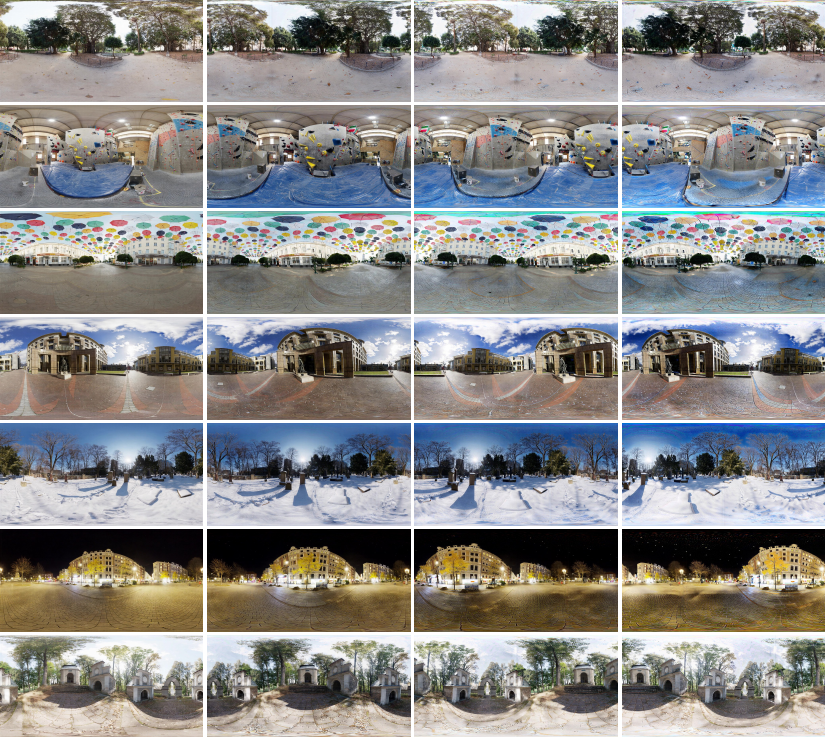}
\caption{
One-minute streams from SPW-Nav. Each row spans 60 to 74 seconds.
}
\label{fig:longvideo}
\end{figure}

\FloatBarrier
\subsection{Navigation accuracy}
Table~\ref{tab:language} evaluates the vision-language backbone in isolation on the held-out test scene and on all navigation instructions of SPW-NavBench.

\begin{table}[!ht]
\centering
\caption{
Navigation accuracy. Accuracy of the vision-language backbone on the held-out test scene and on the navigation instructions of SPW-NavBench.
}
\label{tab:language}
\small
\begin{tabular*}{\linewidth}{@{\extracolsep{\fill}}lcccc@{}}
\toprule
 & \multicolumn{2}{c}{Test scene} & \multicolumn{2}{c}{SPW-NavBench} \\
Method & Action & Goal & Action & Goal \\
\midrule
Template rules & 0.100 & 0.000 & 0.375 & 0.000 \\
\textit{W/o fine-tuning} & 0.663 & \underline{0.154} & \underline{0.850} & \underline{0.450} \\
\textit{W/o image} & \underline{0.900} & 0.135 & \textbf{1.000} & 0.000 \\
\textbf{Ours} & \textbf{0.925} & \textbf{0.692} & \textbf{1.000} & \textbf{0.617} \\
\bottomrule
\end{tabular*}
\end{table}

\FloatBarrier
\subsection{Streaming generation}
Table~\ref{tab:streaming} reports metrics for each of four consecutive instructions in the streaming experiment.

\begin{table}[!ht]
\centering
\caption{
Streaming generation. Metrics over four consecutive instructions in 20 scenes. Since segment 2 is a pure turn, its TDE is undefined.
}
\label{tab:streaming}
\small
\begin{tabular*}{\linewidth}{@{\extracolsep{\fill}}lcccc@{}}
\toprule
Segment & 1 & 2 & 3 & 4 \\
\midrule
TDE (\(^\circ\)) $\downarrow$ & 11.2 & N/A & 31.8 & 19.1 \\
RE (\(^\circ\)) $\downarrow$ & 0.9 & 2.8 & 1.1 & 0.6 \\
MA $\uparrow$ & 0.86 & 1.00 & 0.51 & 0.69 \\
FVD $\downarrow$ & 224.9 & 364.4 & 316.3 & 435.5 \\
\bottomrule
\end{tabular*}
\end{table}

\FloatBarrier
\subsection{Comparison with perspective generators}
Perspective camera-controlled generators take no language instructions and cannot produce panoramas. We therefore compare them only in this setting, on a perspective view cropped from SPW-Nav's output (Table~\ref{tab:perspective}). SPW-Nav has the lowest rotation errors, and its video quality is second only to CameraCtrl. The evaluated views also differ in how they are produced. The \(60^\circ\) crop covers one sixth of the panorama width, about 340 pixels of SPW-Nav's output, while each perspective baseline spends its full output resolution on this single view and generates nothing outside it. HunyuanWorld-Voyager follows translation more accurately through depth-based reprojection.

\begin{table}[!ht]
\centering
\caption{
Comparison with perspective generators on SPW-NavBench.
}
\label{tab:perspective}
\small
\setlength{\tabcolsep}{2pt}
\begin{tabular*}{\linewidth}{@{\extracolsep{\fill}}lccccccc@{}}
\toprule
Method & PSNR\,\(\uparrow\) & LPIPS\,\(\downarrow\) & FVD\,\(\downarrow\) & TDE\,\(\downarrow\) & RE\,\(\downarrow\) & MA\,\(\uparrow\) & RPE\textsubscript{R}\,\(\downarrow\) \\
\midrule
Stable Virtual Camera~\citep{zhou2025seva} & 12.52 & 0.531 & 553.6 & 26.3 & 6.95 & 0.677 & 1.987 \\
ReCamMaster~\citep{bai2025recammaster} & 15.54 & 0.440 & 353.7 & 48.9 & 5.19 & 0.336 & 0.836 \\
CamI2V~\citep{zheng2024cami2v} & 16.54 & 0.364 & 332.8 & 30.4 & 3.29 & 0.666 & \underline{0.483} \\
HunyuanWorld-Voyager~\citep{huang2025voyager} & 16.57 & 0.388 & 301.0 & \textbf{17.1} & 3.45 & \textbf{0.768} & 0.583 \\
CameraCtrl~\citep{he2024cameractrl} & \textbf{17.31} & \textbf{0.340} & \textbf{269.8} & \underline{25.5} & \underline{3.28} & \underline{0.713} & 0.493 \\
Matrix-Game 3.0~\citep{wang2026matrixgame3} & 16.10 & 0.355 & 844.2 & 93.2 & 8.69 & 0.057 & 0.943 \\
\midrule
\textbf{SPW-Nav} & \underline{17.12} & \underline{0.352} & \underline{278.8} & 27.0 & \textbf{2.23} & 0.617 & \textbf{0.321} \\
\bottomrule
\end{tabular*}
\end{table}

\FloatBarrier
\subsection{Guidance weight}
Table~\ref{tab:guidance} sweeps the OTG weight \(\kappa\) with and without NTL.

\begin{table}[!ht]
\centering
\caption{
Full sweep of the OTG weight \(\kappa\) with and without NTL on SPW-NavBench.
}
\label{tab:guidance}
\small
\setlength{\tabcolsep}{2pt}
\begin{tabular*}{\linewidth}{@{\extracolsep{\fill}}cccccccccc@{}}
\toprule
NTL & \(\kappa\) & TDE\,\(\downarrow\) & MA\,\(\uparrow\) & ATE\,\(\downarrow\) & RE\,\(\downarrow\) & RPE\textsubscript{R}\,\(\downarrow\) & PSNR\,\(\uparrow\) & LPIPS\,\(\downarrow\) & FVD\,\(\downarrow\) \\
\midrule
 & 1 & 29.15 & 0.469 & \textbf{4.80} & 26.58 & 3.157 & 13.69 & 0.494 & 512.2 \\
 & 2 & 27.85 & 0.517 & 5.64 & 28.13 & 3.350 & 13.26 & 0.514 & 562.3 \\
\midrule
\checkmark & 1 & 43.53 & 0.395 & 12.33 & \textbf{0.99} & \textbf{0.160} & \textbf{17.05} & \textbf{0.361} & 395.6 \\
\checkmark & 2 & \underline{24.30} & 0.631 & \underline{5.00} & \underline{1.84} & \underline{0.282} & \underline{16.48} & \underline{0.382} & \textbf{308.0} \\
\checkmark & 3 & \textbf{24.09} & \textbf{0.649} & 5.25 & 3.14 & 0.506 & 15.67 & 0.412 & \underline{333.4} \\
\checkmark & 4 & 24.80 & \underline{0.642} & 6.24 & 5.71 & 0.959 & 14.99 & 0.442 & 381.5 \\
\checkmark & 5 & 24.82 & 0.634 & 6.88 & 4.90 & 1.089 & 14.42 & 0.469 & 436.9 \\
\bottomrule
\end{tabular*}
\end{table}

\FloatBarrier
\section{Evaluation protocols}
\label{app:evaluation}

\FloatBarrier
\subsection{Protocol}
The navigation table evaluates the vision-language backbone on a held-out test scene and
on SPW-NavBench, which contains 40 displacement instructions and 60 visible-target
instructions. The streaming and long-horizon experiments use 20 scenes. The long-horizon experiment runs
two closed routes of twelve chunks per scene whose final pose equals the initial one. The first
moves forward, turns around, walks back, and turns around again. The second moves forward and
back, then turns right and left by the same angle. Matrix-Game 3.0 receives each route as keyboard and mouse actions under its official calibration, and CubeComposer completes its perspective video into a panorama for the panoramic evaluation. Unless noted otherwise, video metrics are computed at the
source timestamps after aligning each prediction to the evaluated view.

\FloatBarrier
\subsection{Metric definitions}
We use two groups of metrics. For the video-generation metrics, camera motion is
estimated from forward crops of the generated video with a VGGT-based estimator.
Trajectory direction error (TDE) is the angle between the commanded and estimated
translation directions, rotation error (RE) is the error of the cumulative rotation, and
RPE\textsubscript{R} is the rotation error per interval. \(\Delta\)Rot.\ is the absolute
difference between the mean estimated cumulative rotation and the mean commanded rotation. ATE is a scale-free trajectory shape
error, and motion accuracy (MA) is the fraction of evaluation intervals whose motion follows
the commanded direction. PSNR on the forward view, LPIPS averaged over four viewing directions, and FVD measure
video quality. Memory gain compares
a revisited view with its earlier observation, relative to a non-revisiting view at the same
time gap, and closure errors compare the final estimated pose with the initial one. The navigation-target metrics evaluate the vision-language backbone alone. A
displacement target is correct when its path is within \(\max(0.15\,\text{m}, 10\%)\) of the
commanded excursion and its heading within \(5^\circ\). A goal target is correct when its box
overlaps the annotated target with an IoU of at least 0.3 and its center bearing is within
\(20^\circ\). The Action and Goal columns of the navigation table report these accuracies.

\FloatBarrier
\section{Limitations and future work}
\label{app:limitations}
The multi-term memory of generated frames already keeps revisited views consistent and closes long routes with near-exact rotation. A more explicit memory (\eg, a pose-indexed 3D scene representation) could make consistency on distant revisits more reliable, although advancing memory mechanisms for video generation may also achieve this within the video itself. More diverse training trajectories could further improve the trajectory shape of translation.

\FloatBarrier
\section{Dataset details}
\label{app:dataset}

\textbf{Indoor videos.} We render 8,527 panoramic videos with Habitat from 29 indoor scene configurations, three from HM3D, ten from Replica, and sixteen from ReplicaCAD, along camera paths that cover translation in all four horizontal directions, turns, and composed routes.

\textbf{Outdoor videos.} The outdoor subset contains 9,609 MUGEN clips from 2,856 source videos after removing clips with static cameras or in-place rotation. The original captions often describe turning, which the rotation-free video no longer shows. We therefore replace them with a neutral scene prompt. The distillation stage uses a direction-balanced subset of 2,500 of these videos.

\textbf{Instruction annotations.} Explicit instructions cover translation in four directions, pure turns, stopping, corrections, out-and-back routes, and two-leg routes, with varied distances, English and Chinese phrasings, and varied background panoramas. Each goal is labeled with a bounding box and an approach distance, verified by two vision-language teachers, and augmented by rotating the panorama and its box together. The annotations span 68 unique goals from four indoor scenes. Two scenes provide the training set, and validation and test sets of 128 and 132 instructions come from the two remaining scenes, whose action wording and distances are also absent from training.

\textbf{SPW-NavBench.} The 100 clips come from 94 MUGEN source videos, each with an 81-frame window. Of the navigation instructions, 40 are explicit instructions and 60 are goal instructions.

\textbf{Curation pipeline.} For the outdoor videos, we discard MUGEN clips whose camera stays nearly static and clips that only rotate in place, resample the remaining clips to be rotation-free, and replace their captions with a neutral scene prompt, which leaves 9,609 clips. For the instruction annotations, explicit instructions are derived from camera trajectories, and their motion targets are therefore exact by construction. A goal candidate is kept only after two vision-language teachers verify it.

\textbf{Data splits.} The instruction annotations use disjoint scenes for training, validation, and testing. SPW-NavBench shares neither clips nor source videos with any training data of the video backbone, including the distillation subset, or with the instruction annotations.

\textbf{Statistics.} SPW-NavSet has 18,136 clips (9,609 outdoor, 8,527 indoor), about 233 hours at 30 FPS. Figure~\ref{fig:dataset_samples} shows sample clips.

\begin{figure}[!ht]
\centering
\includegraphics[width=\linewidth,height=0.88\textheight,keepaspectratio]{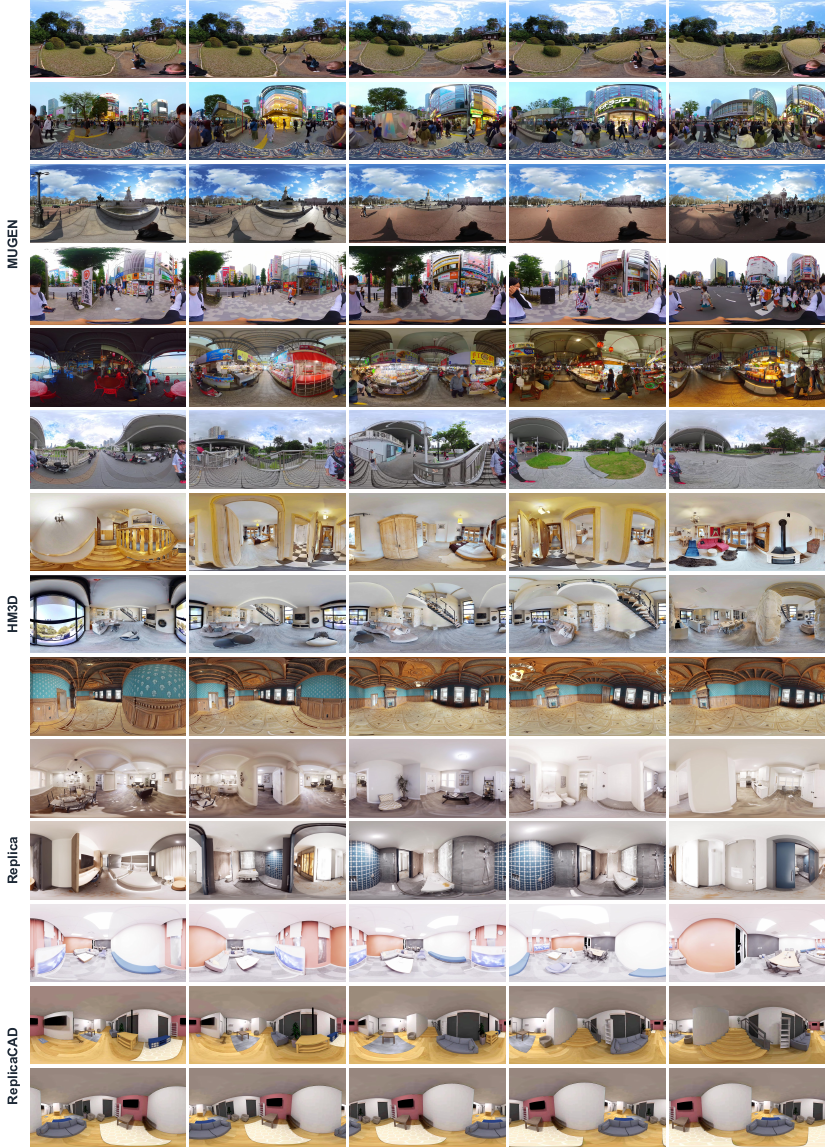}
\caption{
Sample clips of SPW-NavSet. Each row is one clip, with frames evenly spaced from the first to the last frame. The label on the left gives the source of the clips.
}
\label{fig:dataset_samples}
\end{figure}

\FloatBarrier
\section{Implementation details}
\label{app:implementation}

\textbf{Backbone and representation.} The video backbone is Wan2.2-TI2V-5B~\citep{wan2025wan} fine-tuned on panoramic video,
with LoRA of rank 256 on its linear layers. Its VAE compresses frames by $16\times$
spatially and $4\times$ temporally, giving a $64\times128$ latent grid. Decoding pads
the latents by two columns copied from the opposite side, preserving continuity at the
left and right panorama borders. Generated videos have resolution $1024\times2048$ and
16 FPS.

\textbf{Training schedule.} The video backbone is trained without language on the rotation-free videos in three stages, autoregressive training with the multi-term
memory (10,000 steps), pyramid training (4,000 and 9,500 steps), and
trajectory-conditioned training (28,832 steps). The student is initialized by 10,000 ODE-regression
steps and distilled for 3,500 steps, 1,500 over the full noise range followed by 2,000 with noise-tail sampling, with learning rates of
$2\times10^{-6}$ (generator) and $4\times10^{-7}$ (critic), five critic
updates per generator step, and an EMA decay of 0.99.

\textbf{Vision-language training and compute.} The vision-language backbone is fine-tuned for three epochs with learning rate
$5\times10^{-5}$ and batch size 16 on four NVIDIA H200 GPUs.

\FloatBarrier
\section{Baseline details}
\label{app:baselines}

\textbf{Panoramic interfaces.} NWM~\citep{bar2024navigation} predicts $224\times224$ perspective views from navigation
actions and additionally receives three ground-truth past views, as required by its
interface. CubeComposer~\citep{li2026cubecomposer} completes a perspective video into a
panorama. Since it takes no trajectory, we repeat the first panorama as a static clip and
render its perspective input along the commanded rotation. It thus receives no translation.
OmniRoam~\citep{liu2026omniroam} takes the full camera trajectory. The Action and Causal
variants of PanoWorld~\citep{li2026panoworld} fix the heading and receive translation
only. All panoramic baselines generate 81 frames at $960\times480$, except CubeComposer,
which generates 27 frames at $1536\times3072$. All are evaluated at matching timestamps.

\textbf{Perspective interfaces.} Stable Virtual Camera~\citep{zhou2025seva}, ReCamMaster~\citep{bai2025recammaster},
CamI2V~\citep{zheng2024cami2v}, HunyuanWorld-Voyager~\citep{huang2025voyager}, and
CameraCtrl~\citep{he2024cameractrl} receive the same $60^\circ$ crop of the first
panorama as the conditioning image. Each output is aligned to its actual field of view
before evaluation.

\textbf{Language interfaces.} In the end-to-end setting, PanoWorld and OmniRoam receive the instruction appended to
their text prompts together with the native motion preset matching the instruction's
keywords. NWM has no text input and receives the actions planned by our vision-language
backbone without fine-tuning. CubeComposer takes no language and is excluded.

\FloatBarrier
\section{Streaming generation loop}
\label{app:algorithm}
Algorithm~\ref{alg:streaming} makes the state update explicit. Each instruction is interpreted once, while generation, memory update, and exact spherical rotation happen once per chunk.

\begin{algorithm}[H]
\caption{Streaming panoramic generation}
\label{alg:streaming}
\small
\begin{algorithmic}
\REQUIRE Initial panorama $I_0$ at pose $T_0$, instructions $\{u_1, \ldots, u_K\}$
\STATE $\mathcal{M} \gets \text{Bootstrap}(I_0)$, $T_{\text{cur}} \gets T_0$, $\bar{I} \gets I_0$
\FOR{$k = 1$ \textbf{to} $K$}
  \STATE $A_k \gets \text{Trajectory}(\text{Interpret}(u_k, \bar{I}), T_{\text{cur}})$
  \FOR{each chunk $c$ of $A_k$}
    \STATE $\mathbf{p}_c \gets$ translations of chunk $c$ relative to its reference $\rho_c$
    \STATE $\widetilde{V}_c \gets \text{PyramidGenerate}(\mathcal{M}, \mathbf{p}_c)$ with Eq.~\ref{eq:guidance} at the coarsest scale
    \STATE $\mathcal{M} \gets \text{Update}(\mathcal{M}, \widetilde{V}_c)$, $V_c \gets \text{Rotate}(\text{Decode}(\widetilde{V}_c), A_k)$
  \ENDFOR
  \STATE $T_{\text{cur}} \gets T_{k,L}$, $\bar{I} \gets$ last frame of $V_c$
\ENDFOR
\end{algorithmic}
\end{algorithm}

\section{Correctness of spherical rotation decoupling}
\label{app:srd}

We show that spherical rotation decoupling reproduces the video of the full camera trajectory exactly, up to the interpolation of discrete resampling. We model the scene as static during a chunk and describe it by its light field \(L(\mathbf{c}, \mathbf{w})\), the color seen from a camera center \(\mathbf{c} \in \mathbb{R}^3\) in the world direction \(\mathbf{w} \in \mathbb{S}^2\). A camera with pose \(T = (R, \mathbf{t})\), where \(R \in \mathrm{SO}(3)\) maps camera axes to world axes, records the panorama
\begin{equation}
I(\mathbf{d}) = L(\mathbf{t}, R\,\mathbf{d}), \qquad \mathbf{d} \in \mathbb{S}^2,
\label{eq:lightfield}
\end{equation}
where \(\mathbf{d}\) is the ray of a pixel in camera coordinates (Eq.~\ref{eq:erp}). A panorama therefore records every world direction, and the camera orientation only decides which pixel stores which direction.

\begin{lemma}[Rotation is resampling]
\label{lem:rotation}
Let two cameras share the center \(\mathbf{t}\) and have orientations \(R_a\) and \(R_b\). Their panoramas satisfy \(I_b(\mathbf{d}) = I_a(R\,\mathbf{d})\) for all \(\mathbf{d} \in \mathbb{S}^2\), with \(R = R_a^\top R_b \in \mathrm{SO}(3)\).
\end{lemma}
\begin{proof}
By Eq.~\ref{eq:lightfield}, \(I_b(\mathbf{d}) = L(\mathbf{t}, R_b\,\mathbf{d}) = L(\mathbf{t}, R_a R_a^\top R_b\,\mathbf{d}) = I_a(R_a^\top R_b\,\mathbf{d})\), since \(R_a R_a^\top = I\).
\end{proof}

Since \(R\) is a bijection of \(\mathbb{S}^2\) onto itself, the resampling neither creates nor removes content. A pure rotation of a panoramic camera thus requires no generation, in contrast with a perspective camera, whose field of view covers only part of the sphere.

\begin{proposition}[Decoupling]
\label{prop:srd}
Let a video \(\{I_n\}\) be recorded along poses \(T_n = (R_n, \mathbf{t}_n)\), and let \(\widetilde{I}_n(\mathbf{d}) = I_n(R_n^\top R_0\,\mathbf{d})\) and \(\widetilde{\mathbf{t}}_n = R_0^\top(\mathbf{t}_n - \mathbf{t}_0)\) as in Section~\ref{sec:method:generation}. Then
\begin{enumerate}
\item[(i)] \(\widetilde{I}_n\) is the panorama recorded at center \(\mathbf{t}_n\) with the fixed orientation \(R_0\), so \(\{\widetilde{I}_n\}\) is a physically valid video of the same scene along a pure translation, which is \(\widetilde{\mathbf{t}}_n\) in the camera axes of the first frame.
\item[(ii)] the original video is recovered exactly as \(I_n(\mathbf{d}) = \widetilde{I}_n(R_0^\top R_n\,\mathbf{d})\).
\item[(iii)] the rotation-free video depends on the trajectory only through \(\{\widetilde{\mathbf{t}}_n\}\), so any two trajectories with the same initial pose and translations but different subsequent rotations share the same rotation-free video.
\end{enumerate}
\end{proposition}
\begin{proof}
(i) By Eq.~\ref{eq:lightfield}, \(\widetilde{I}_n(\mathbf{d}) = L(\mathbf{t}_n, R_n R_n^\top R_0\,\mathbf{d}) = L(\mathbf{t}_n, R_0\,\mathbf{d})\), which is the panorama of a camera at \(\mathbf{t}_n\) with orientation \(R_0\). Its displacement from the first frame is \(\mathbf{t}_n - \mathbf{t}_0\) in world axes, or \(R_0^\top(\mathbf{t}_n - \mathbf{t}_0) = \widetilde{\mathbf{t}}_n\) in the axes of a camera with orientation \(R_0\).
(ii) Applying Lemma~\ref{lem:rotation} with \(R_a = R_0\) and \(R_b = R_n\) at the common center \(\mathbf{t}_n\) gives \(I_n(\mathbf{d}) = \widetilde{I}_n(R_0^\top R_n\,\mathbf{d})\). Directly, \(\widetilde{I}_n(R_0^\top R_n\,\mathbf{d}) = L(\mathbf{t}_n, R_0 R_0^\top R_n\,\mathbf{d}) = L(\mathbf{t}_n, R_n\,\mathbf{d}) = I_n(\mathbf{d})\).
(iii) By (i), \(\widetilde{I}_n(\mathbf{d}) = L(\mathbf{t}_n, R_0\,\mathbf{d})\) contains no \(R_n\), and \(\mathbf{t}_n = \mathbf{t}_0 + R_0\,\widetilde{\mathbf{t}}_n\) is determined by \(\widetilde{\mathbf{t}}_n\) once the first pose is fixed.
\end{proof}

Proposition~\ref{prop:srd} justifies both halves of spherical rotation decoupling. By (i), the rotation-free training videos are genuine translation-only videos of the original scenes, so the video backbone learns real parallax rather than an artificial warp. By (ii) and (iii), a model that generates the rotation-free video for the commanded translations and then applies \(R_0^\top R_n\) frame by frame produces the video of the full trajectory, and the rotation never needs to be generated.

\begin{proposition}[Invariance of the pose inputs]
\label{prop:pac}
Let \(G \in \mathrm{SE}(3)\) be any rigid transformation of the world frame, which maps every pose \(T_n\) to \(G\,T_n\). The relative pose of pose-aligned conditioning is unchanged, \((G\,T_{\rho_k})^{-1}(G\,T_n) = T_{\rho_k}^{-1} T_n\). For rotation-free poses, its rotation block is the identity.
\end{proposition}
\begin{proof}
\((G\,T_{\rho_k})^{-1}(G\,T_n) = T_{\rho_k}^{-1} G^{-1} G\,T_n = T_{\rho_k}^{-1} T_n\). All rotation-free poses share the orientation \(R_0\), so the relative rotation is \(R_0^\top R_0\), the identity.
\end{proof}

\begin{remark}
In practice, \(\mathbf{d} \mapsto R\,\mathbf{d}\) is evaluated at the pixel rays of Eq.~\ref{eq:erp} with bilinear sampling that wraps around the longitude seam, so the equalities above hold up to interpolation error. For a rotation about the vertical axis by an angle \(\alpha\), the resampling reduces to a circular horizontal shift \(I'(\theta, \phi) = I(\theta + \alpha, \phi)\), which is exact whenever \(\alpha\) is a multiple of the pixel width \(2\pi/W\). We apply the rotation to decoded frames rather than to latents, because the video autoencoder is not equivariant to general rotations of the sphere, while generation and memory stay in the initial orientation.
\end{remark}

\end{document}